\documentclass[10pt, twocolumn, a4paper]{article}

\usepackage{microtype}
\usepackage{graphicx}
\usepackage{subfigure}
\usepackage{booktabs} 
\usepackage{hyperref}
    \hypersetup{colorlinks=true, linkcolor=black, citecolor=black, urlcolor=black}
\usepackage[top=2cm, bottom=2cm, left=1.5cm, right=1.5cm]{geometry}
\usepackage{fancyhdr}
\usepackage{titlesec}
    \titleformat{\title}{\large\bfseries}{}{}{}
    \titleformat{\section}{\normalfont\bfseries}{\thesection}{0.5em}{}
    \titleformat{\subsection}{\normalfont\it}{\thesubsection}{0.5em}{}
    \titleformat{\subsubsection}{\normalfont\normalsize\it}{\thesubsubsection}{0.5em}{}
    \titleformat{\paragraph}[runin]{\normalfont\bfseries}{\theparagraph}{0.5em}{}
    \titleformat{\subparagraph}[runin]{\normalfont\normalsize\it}{\thesubparagraph}{0.5em}{}
\usepackage[font=small,labelfont=bf,labelsep=space]{caption}
\usepackage[shortlabels, inline]{enumitem} 
\usepackage[round]{natbib}
	\renewcommand{\cite}[1]{\citep{#1}}
\usepackage{bbold}
\usepackage{amsthm, amsmath}
\usepackage{tikz}

\DeclareMathOperator*{\argmin}{arg\,min}

\makeatletter
\providecommand*{\cupdot}{%
	\mathbin{%
		\mathpalette\@cupdot{}%
	}%
}
\newcommand*{\@cupdot}[2]{%
	\ooalign{%
		$\m@th#1\cup$\cr
		\hidewidth$\m@th#1\cdot$\hidewidth
	}%
}
\makeatother

\newtheorem{theorem}{Theorem}
\newtheorem{corollary}{Corollary}
\newtheorem{lemma}{Lemma}
\newtheorem{definition}{Definition}
\newtheorem{remark}{Remark}

\begin{document}

\author{David Stein and Bjoern Andres$^1$}
\title{\bf Inapproximability of a Pair of Forms Defining a\\Partial Boolean Function}
\date{TU Dresden}

\twocolumn[
\begin{@twocolumnfalse}
\maketitle
\begin{abstract}
We consider the problem of jointly minimizing forms of two Boolean functions $f, g \colon \{0,1\}^J \to \{0,1\}$ such that $f + g \leq 1$ and so as to separate disjoint sets $A \cupdot B \subseteq \{0,1\}^J$ such that $f(A) = \{1\}$ and $g(B) = \{1\}$.
We hypothesize that this problem is easier to solve or approximate than the well-understood problem of minimizing the form of one Boolean function $h: \{0,1\}^J \to \{0,1\}$ such that $h(A) = \{1\}$ and $h(B) = \{0\}$.
For a large class of forms, including binary decision trees and ordered binary decision diagrams, we refute this hypothesis.
For disjunctive normal forms, we show that the problem is at least as hard as \textsc{min-set-cover}.
For all these forms, we establish that no $o(\ln (|A| + |B| -1))$-approximation algorithm exists unless $\textsc{p=np}$.
\end{abstract}

\vspace{7ex}
\end{@twocolumnfalse}
]

\maketitle

\footnotetext[1]{Correspondence: \href{mailto:bjoern.andres@tu-dresden.de}{\texttt{bjoern.andres@tu-dresden.de}}}

\section{Introduction}
The desire to apply machine learning in safety-critical environments has renewed interest in the learning of partial functions.
In medicine, for instance, doctors may require a partial function to distinguish between positive findings (1), negative findings (0) and findings prioritized for human inspection (-).
In the field of autonomous driving, engineers may require a partial function to distinguish between autonomous driving mode (1), emergency breaking (0) and escalation to the driver (-).
In this article, we contribute to the understanding of the hardness of learning partial functions.

Specifically, we concentrate on partial Boolean functions.
While a (total) Boolean function $f \colon \{0,1\}^J \to \{0,1\}$ defines a decision, $f(x) \in \{0,1\}$ for every assignment $x \in \{0,1\}^J$ of zeroes or ones to the finite, non-empty set $J$ of input variables, a partial Boolean function, i.e.~a map from a subset of $\{0,1\}^J$ to $\{0,1\}$, distinguishes between positive, negative and undecided inputs $x$ by $f(x)$ being either 1, 0 or undefined.
By the examples above, we have seen that such distinctions are relevant in safety-critical environments.

Our work is motivated by the hypothesis that the problem of learning a partial Boolean function is easier to solve or approximate than the problem of learning a (total) Boolean function.
Intuition might lead us to speculate that the hypothesis is true because there is more freedom in choosing a partial Boolean function than there is in choosing a (total) Boolean function.
Anyhow, we understand that the trueness of the hypothesis can depend on the encoding of partial Boolean functions as well as on the learning problem.

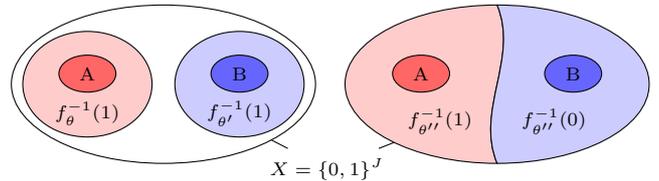
\begin{figure}[t]
\centering
\begin{tikzpicture}[xscale=0.5, yscale=0.7]
\tikzstyle{every node} = [font=\scriptsize]
\node[fill=white] at (4.3, -1.6) (n) {$X = \{0,1\}^J$};
\draw (0, 0) -- (n);
\draw (10, 0) -- (n);
\begin{scope}
	\draw[fill=white] (0, 0) ellipse (4 and 1.5);
	
	\draw[fill=red!20] (-2, 0) ellipse (1.7 and 1.0);
	\draw[fill=red!60] (-2, 0.2) ellipse (0.75 and 0.35);
	\node at (-2, 0.2) {A};
	\node at (-2, -0.55) {$f_\theta^{-1}(1)$};
	
	\draw[fill=blue!20] (2, 0) ellipse (1.7 and 1.0);
	\draw[fill=blue!60] (2, 0.2) ellipse (0.75 and 0.35);
	\node at (2, 0.2) {B};
	\node at (2, -0.55) {$f_{\theta'}^{-1}(1)$};
\end{scope}
\begin{scope}[xshift=58ex]
	\draw[fill=red!20] (90:4 and 1.5) arc(90:270:4 and 1.5) to[out=120, in=-60] (0, 1.5);
	\draw[fill=blue!20] (270:4 and 1.5) arc(270:450:4 and 1.5) to[out=-60, in=120] (0, -1.5);
	
	\draw[fill=red!60] (-2, 0.2) ellipse (0.75 and 0.35);
	\draw[fill=blue!60] (2, 0.2) ellipse (0.75 and 0.35);
	
	\node at (-1.5, -0.7) {$f_{\theta''}^{-1}(1)$};
	\node at (1.5, -0.7) {$f_{\theta''}^{-1}(0)$};
	\node at (-2, 0.2) {A};
	\node at (2, 0.2) {B};
\end{scope}
\end{tikzpicture}
\caption{We study the problem (depicted on the left) of jointly minimizing forms $\theta, \theta'$ of Boolean functions $f_\theta, f_{\theta'} \colon \{0,1\}^J \to \{0,1\}$ such that $f_\theta + f_{\theta'} \leq 1$ and so as to separate disjoint sets $A \cup B \subseteq \{0,1\}^J$ such that $f_\theta(A) = \{1\}$ and $f_{\theta'}(B) = \{1\}$.
Note that $f_\theta$ and $f_{\theta'}$ together define a partial Boolean function that assumes the value one (zero) iff $f_\theta$ ($f_{\theta'}$) assumes the value one.
Contrast this problem with the well-known problem (depicted on the right) of separating $A$ and $B$ by a single Boolean function $f_{\theta''}$ such that $f_{\theta''}(A) = \{1\}$ and $f_{\theta''}(B) = \{0\}$.}
\label{fig:problemsetvis}
\end{figure}

Here, we choose to encode any partial Boolean function by a pair $(f,g)$ of Boolean functions $f,g \colon \{0,1\}^J \to \{0,1\}$ such that $f + g \leq 1$.
Together, $f$ and $g$ define the unique partial Boolean function that assumes the value one iff $f$ assumes the value one, and assumes the value zero iff $g$ assumes the value one.
See also Fig.~\ref{fig:problemsetvis}.
More specifically, we choose to encode $f$ and $g$ both by the same type of form, e.g.~a binary decision tree, ordered binary decision diagram (OBDD) or disjunctive normal form (DNF).
As a learning problem, we consider the objective of minimizing a sum of complexities of these forms, subject to the additional constraint that disjoint sets $A \cupdot B = S \subseteq \{0,1\}^J$ are separated such that $f(A) = \{1\}$ and $g(B) = \{1\}$.
This problem is depicted also in Fig.~\ref{fig:problemsetvis}.

We show:
For a large class of forms including binary decision trees and OBDDs, this problem is at least as hard as the well-understood problem, also depicted in Fig.~\ref{fig:problemsetvis}, of minimizing the form of one Boolean function $h: \{0,1\}^J \to \{0,1\}$ such that $h(A) = \{1\}$ and $h(B) = \{0\}$.
For the class of DNFs, the problem we introduce is at least as hard as \textsc{min-set-cover}.
For binary decision trees, OBDDs and DNFs, no $o(\ln (|A| + |B| -1))$-approximation algorithm exists unless $\textsc{p=np}$.

\section{Related Work}
The problem of extending a partial Boolean function defined by a set of true points $A \subseteq\{0, 1\}^J$ and a disjoint set of false points $B \subseteq \{0, 1\}^J$ to a total function $h: \{0, 1\}^J \to \{0, 1\}$ such that $A \subseteq h^{-1}(1)$ and $B \subseteq h^{-1}(0)$ has been studied comprehensively, for various classes of functions~\cite{Crama2011}. 
In particular, deciding whether a binary decision tree or DNF or OBDD of bounded complexity exists which classifies the set of truth points and false points exactly is \textsc{np}-complete \cite{Lukas1999,Hancock1996,sauerhoff-1996,Haussler1988}. 
The problem of finding a DNF of bounded length remains \textsc{np}-hard even if the full truth table is given as input~\cite{Allender2008,Lukas1999}. 

Toward approximation, the problem of finding a DNF of minimum depth or length, consistent with labeled data, does not admit a polynomial-time $o(\ln (|A| + |B| - 1))$-approximation algorithm unless
\textsc{p=np}, due to an approximation-preserving reduction of \textsc{min-set-cover} by \citet{Lukas1999,Feige1998} and the inapproximability of \textsc{min-set-cover} established by \citet{irit-2014}.
See also \citet{moshkovitz-2015}.

Analogously, the problem of finding a binary decision tree of minimum depth or number of nodes, consistent with labeled data, does not admit a polynomial-time $o(\ln (|A| + |B| - 1))$-approximation algorithm unless $\textsc{p=np}$, by the approximation-preserving reduction of \textsc{min-set-cover} by \citet{Hancock1996} and the inapproximability of \textsc{min-set-cover} due to \citet{irit-2014}.

Analogously still, the problem of finding an OBDD with a minimum number of interior nodes, consistent with labeled data \citep{Hirata1996,sauerhoff-1996,takenaga-2000}, does not admit a polynomial-time $o(\ln (|A| + |B| - 1))$-approximation algorithm unless $\textsc{p=np}$, by the approximation-preserving reduction of \textsc{min-set-cover} by \citet{Hirata1996} and the inapproximability of \textsc{min-set-cover} due to \citet{irit-2014}. 

The related problem of isolating points by binary decision trees is \textsc{np}-hard \cite{Hyafil1976}. 
It does not admit a polynomial time $o(\ln (|A| + |B| - 1))$-approximation algorithm unless $\textsc{p=np}$ by the approximation-preserving reduction of \textsc{min-set-cover} by \citet{Laber2004} and the inapproximability of \textsc{min-set-cover} due to \citet{irit-2014}.

The related problem of deciding, for any binary decision tree given as input, whether an equivalent binary decision tree of size at most $k\in \mathbb{N}$ exists is \textsc{np}-complete~\cite{ZANTEMA2000}. 
The corresponding optimization problem of finding an equivalent binary decision tree of minimal size does not admit a polynomial-time $r$-approximation algorithm for any constant $r>1$, unless \textsc{$\textsc{p=np}$} \cite{Sieling2008}.

The related problem of deciding, for any OBBD given as input, whether an equivalent OBDD of size at most $k\in \mathbb{N}$ exists is \textsc{np}-complete~\cite{Bollig1996}. 
The corresponding optimization problem of finding an equivalent OBDD of minimal size does not admit a polynomial-time $r$-approximation algorithm for any constant $r> 1$, unless $\textsc{p=np}$~\cite{Sieling2002}.

\section{Problem Statement}
The problems we study in this article are about learning from data with binary features and binary labels.
For conciseness, we call a tuple $(J, X, A, B)$ \emph{Boolean labeled data} with the \emph{feature space} $X$ iff $J \neq \emptyset$ is finite and $X = \{0,1\}^J$ and $A \cup B \subseteq X$ and $A \neq \emptyset$ and $B \neq \emptyset$ and $A \cap B = \emptyset$.
For any $x \in A$, we say that $x$ is \emph{labeled} $A$.
For any $x \in B$, we say that $x$ is \emph{labeled} $B$.
For all $x \in X \setminus (A \cup B)$, we say that $x$ is \emph{unlabeled}.
We study the following problems:

\begin{definition}
\label{definition:problems}
For any Boolean labeled data $(J, X, A, B) = D$,
any non-empty family $f \colon \Theta \to \{0, 1\}^X$ of Boolean functions and 
any $R \colon \Theta \to \mathbb{N}_0$ called a \emph{regularizer},
the instance of \textsc{partial-separation} wrt.~$D$, $\Theta$, $f$ and $R$ has the form
\begin{align}
\min_{(\theta, \theta')\in \Theta^2} \quad & R(\theta) + R(\theta') & \\
\textnormal{subj.~to} \quad 
&\forall x \in A \colon \ \ f_{\theta}(x) = 1 & \hspace{-3.5ex}\textnormal{(exactness)}\\
&\forall x \in B \colon \ \ f_{\theta'}(x) = 1 & \hspace{-3.5ex}\textnormal{(exactness)}\\
&f_{\theta} + f_{\theta'} \leq 1 & \hspace{-3.5ex} \hspace{-2ex}\textnormal{(non-contradictoriness)}
\end{align}

For any additional $m \in \mathbb{N}_0$, the instance of \textsc{partial-separability} wrt.~$D$, $\Theta$, $f$, $R$ and $m$ is to decide whether there exist $\theta, \theta' \in \Theta$ such that the following conditions hold:
\begin{align}
R(\theta) + R(\theta') & \leq m & \textnormal{(boundedness)}\\
\forall x \in A \colon \quad f_\theta(x) & = 1 & \textnormal{(exactness)}\\
\forall x \in B \colon \quad f_{\theta'}(x) & = 1 & \textnormal{(exactness)}\\
f_\theta + f_{\theta'} & \leq 1 & \textnormal{(non-contradictoriness)}
\end{align}
\end{definition}

\begin{remark}
\begin{enumerate}[label=\alph*)]
\item 
Any feasible solution $(\theta, \theta')\in \Theta^2$ defines a partial Boolean function $h$ from $X$ to $\{0,1\}$ with the domain $X' = f^{-1}_\theta(1) \cup f^{-1}_{\theta'}(1)$ and such that for all $x \in X'$, we have $h(x) = 1$ iff $f_\theta(x) = 1$ and, equivalently, $h(x) = 0$ iff $f_{\theta'}(x) = 1$.
\item 
Exactness means the labeled data is classified totally and without errors, i.e., for all $x\in A$ we have $x\in X'$ and $h(x) = 1$, and for all $x\in B$ we have $x\in X'$ and $h(x) = 0$.
\item 
The problems are symmetric in the sense that $(\theta, \theta') \in \Theta^2$ is a (feasible) solution to an instance wrt.~Boolean labeled data $(J, X, A, B)$ iff $(\theta', \theta)$ is a (feasible) solution to the same instance but with Boolean labeled data labeled data $(J, X, B, A)$.
\item 
Totality can be enforced by the additional constraint 
\begin{align}
1 \leq f_\theta + f_{\theta'}
\enspace .
\end{align}
In this sense, the problem of learning partial Boolean functions is a relaxation of the problem of learning (total) Boolean functions. 
\end{enumerate}
\end{remark}

Contrast Def.~\ref{definition:problems} with the well-understood problems (Def.~\ref{definition:problems-classic} below) of separating two sets, $A$ and $B$, by a single Boolean function $f_\theta$ such that $f_\theta(A) = \{1\}$ and $f_\theta(B) = \{0\}$.
The difference between Def.~\ref{definition:problems} and Def.~\ref{definition:problems-classic} is depicted in Fig.~\ref{fig:problemsetvis}.

\begin{definition}[e.g.~\citet{Haussler1988,Hancock1996,sauerhoff-1996,Feige1998,Lukas1999}]
	\label{definition:problems-classic}
	Let $D = (J, X, A, B)$ be Boolean labeled data.
	Let $f \colon \Theta \to \{0, 1\}^X$ be a non-empty family of Boolean functions, and 
	let $R \colon \Theta \to \mathbb{N}_0$ be called a \emph{regularizer}.
	The instance of \textsc{separation} wrt.~$D$, $\Theta$, $f$ and $R$ has the form
	\begin{align}
		\min_{\theta\in \Theta} \quad & R(\theta) & \\
		\textnormal{subject to} \quad & 
		\forall x \in A \colon \quad f_{\theta}(x) = 1 \\
		& \forall x \in B \colon \quad f_{\theta}(x) = 0 
	\end{align}
	
	Let $m\in \mathbb{N}_0$ in addition. The instance of \textsc{separability} wrt.~$D$, $\Theta$, $f$, $R$ and $m$ is to decide whether there exists a $\theta \in \Theta$ such that the following conditions hold:
	\begin{align}
		R(\theta) & \leq m \\
		\forall x \in A \colon \quad f_\theta(x) & = 1 \\
		\forall x \in B \colon \quad f_\theta(x) & = 0
	\end{align}
\end{definition}

Our work is motivated by the hypothesis that the problems according to Def.~\ref{definition:problems} are easier to solve or approximate than the problems according to Def.~\ref{definition:problems-classic}.
This hypothesis is non-trivial in light of Remark~1d.

\section{Preliminaries}
\subsection{Set Cover Problem}

\begin{definition}
	For any finite, non-empty set $U$, any collection $\Sigma \subseteq 2^{U}$ of subsets of $U$ and any $m\in \mathbb{N}_0$, the instance of \textsc{set-cover} w.r.t.~$U$, $\Sigma$ and $m$ is to decide whether there exists a $\Sigma'\subseteq \Sigma$ such that $\bigcup_{\sigma\in \Sigma'}\sigma = U$ and
	$\vert\Sigma'\vert \leq m$.
	The instance of \textsc{min-set-cover} w.r.t.~$U$ and $\Sigma$ has the form
	\begin{align}
	\min_{\Sigma'\subseteq \Sigma} \quad&\vert \Sigma'\vert \\
	\textnormal{subject to}\quad& U = \bigcup_{\sigma\in \Sigma'} \sigma\enspace.
	\end{align} 
\end{definition}

\begin{theorem}[\citet{irit-2014}]
	\label{theorem:steurersetcoverbound}
	For every $\epsilon > 0$, it is \textsc{np}-hard to approximate \textsc{min-set-cover} to within $(1-\epsilon)\ln |U|$, where $n$ is the size of the instance.
\end{theorem}

\subsection{Disjunctive Normal Forms}

\begin{definition}
	\label{def:dnfs}
	For any finite, non-empty set $X = \{0, 1\}^J$, 
	for the sets $\Gamma = \left\{\left(V, \bar{V}\right)\in 2^J\times 2^J\vert V\cap \bar{V} = \emptyset\right\}$
	and
	$\Theta = 2^{\Gamma}$,
	the family $f: \Theta \to \{0, 1\}^X$ such that for any $\theta \in \Theta$ and any $x \in X$,
	\begin{equation}
	f_{\theta}(x) = \sum_{(J_0, J_1)\in \theta} \prod_{j\in J_0} x_j \prod_{j\in J_1} \left(1-x_j\right) 
	\end{equation}
	is called the family of $J$-variate \emph{disjunctive normal forms (DNFs)}.
	For $R_l,R_d\colon \Theta \to \mathbb{N}_0$ such that for all $\theta\in \Theta$,
	\begin{align}
	R_l(\theta) &= \sum_{(J_0, J_1)\in \theta} \left(\vert J_0\vert + \vert J_1\vert\right) \\
	R_d(\theta) &= \max_{(J_0, J_1) \in \theta}\left(\vert J_0\vert + \vert J_1\vert\right)
	\end{align}
	$R_l(\theta)$ and $R_d(\theta)$ are called the \emph{length} and \emph{depth}, respectively, of the DNF defined by $\theta$.
\end{definition}

%

\section{Hardness}
\subsection{Polynomially Negatable Functions}

In this section, we show for a large class of families of functions, including binary decision trees and OBDDs, and for typical regularizers, that \textsc{partial-separability} is at least as hard as \textsc{separability}.
For these, it follows from the \textsc{np}-hardness of \textsc{separability} \citep{Haussler1988,takenaga-2000} that \textsc{partial-separability} is also \textsc{np}-hard and, in fact, \textsc{np}-complete.
Beyond the reduction of \textsc{separability} to \textsc{partial-separability}, we already relate feasible solutions as shown in Fig.~\ref{fig:reduction-polynomially-negatable-mapping} in a way that will allow us in the next section to transfer an inapproximability bound from \textsc{separation} to \textsc{partial-separation}.

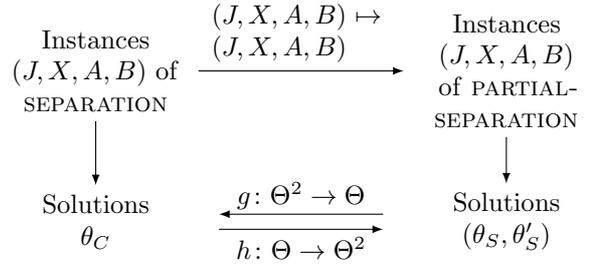
\begin{figure}[t]
	\hspace{-3ex}
	\begin{tikzpicture}[xscale=2.7,yscale=-1]
		\node[text width=18ex, align=center] at (0, 0) (icov) {Instances $(J, X, A, B)$ of \textsc{separation}};
		\node[text width=15ex, align=center] at (2, 0) (isep) {Instances $(J, X, A, B)$ of \textsc{partial-separation}};
		\node[text width=21ex, align=center] at (0, 2) (scov) {Solutions\\$\theta_C$};
		\node[text width=23ex, align=center] at (2, 2) (ssep) {Solutions\\$(\theta_S,\theta_S')$};
		\draw[-latex] (0.5, 0) -- (1.5, 0);
		\node[text width=15ex] at (1, -0.5) {$(J, X, A, B) \mapsto (J, X, A, B)$};
		\draw[-latex] (1.4, 1.9) -- (0.6, 1.9);
		\node at (1, 1.65) {$g \colon \Theta^2 \to \Theta$};
		\draw[-latex] (0.6, 2.1) -- (1.4, 2.1);
		\node at (1, 2.35) {$h \colon \Theta \to \Theta^2$};
		\draw[-latex] (icov) -- (scov);
		\draw[-latex] (isep) -- (ssep);
	\end{tikzpicture}
	\caption{For polynomially negatable functions, we reduce \textsc{separation} to \textsc{partial-separation} in an approximation preserving manner by mapping instances of \textsc{separation} to instances of \textsc{partial-separation} and by relating solutions to these instances by the functions $g$ and $h$.}
	\label{fig:reduction-polynomially-negatable-mapping}
\end{figure}

\begin{definition}
	\label{def:simplenegation}
	Let $J$ finite and non-empty and $X = \{0, 1\}^J$.
	Let $f: \Theta \to \{0, 1\}^X$ non-empty, and
	let $R: \Theta \to \mathbb{N}_0$.
	We call the family $f$ \emph{polynomially negatable under $R$} iff there is a function $n: \Theta \to \Theta$, computable in polynomial time, such that the following conditions hold:
	\begin{align}
		\forall \theta \in \Theta \colon \quad &f_{n(\theta)} + f_\theta = 1 \label{eq:negation}\\
		\forall \theta \in \Theta \colon \quad & R(n(\theta)) \leq R(\theta)
		\label{eq:negation-bound}
	\end{align}
\end{definition}

\begin{lemma}
\label{lemma:negatable-bdt-obdd}
	The family of binary decision trees is polynomially negatable under the number of nodes as well as under the depth of a tree.
	The family of OBDDs is polynomially negatable under the number of interior nodes as well as under the width of the diagram.	
\end{lemma}

\begin{proof}
	For any binary decision tree or OBDD $\theta$, a suitable $n(\theta)$ is obtained by flipping the labels 0 and 1 at those nodes that do not have outgoing edges. 
\end{proof}

\begin{lemma}
	\label{lemma:functionsimplenegatableh}
	For any instance $(J, X, A, B), \Theta, f, R$ of \textsc{separation}, where the family of functions $f$ is polynomially negatable under $R$, any function $n: \Theta \to \Theta$ according to Def.~\ref{def:simplenegation} and the instance of \textsc{partial-separation} wrt.~$(J, X, A, B), \Theta, f$ and $R$, the function 
	\begin{align}
	h : \quad \Theta & \to \Theta^2\nonumber\\
	\theta & \mapsto (\theta, n(\theta))
	\end{align}
	has the following properties: For any feasible solution $\theta$ to the instance of \textsc{separation}:
	\begin{enumerate}
		\item $h(\theta)$ is computable in polynomial time.
		\item $h(\theta) = (\theta, n(\theta))$ is such that $R(\theta) + R(n(\theta)) \leq 2R(\theta)$.
		\item $h(\theta)$ is a feasible solution to the instance of \textsc{partial-separation}.	
	\end{enumerate}
\end{lemma}

\begin{proof}
	(1) and (2) hold by construction of $h$ and Def.~\ref{def:simplenegation}.
	
	(3) $f_\theta$ and $f_{n(\theta)}$ are such that $f_\theta + f_{n(\theta)} \leq 1$ and such that $f_\theta(A) = \{1\}$ and $f_\theta(B) = \{0\}$ as $\theta$ is a feasible solution to the instance of \textsc{separation}. 
	From this follows $f_{n(\theta)}(B) = \{1\}$, by \eqref{eq:negation}.
\end{proof}

\begin{corollary}
	\label{cor:forwardnphardnesssimplenegatable}
	For any instance $(J, X, A, B), \Theta, f, R, m$ of \textsc{separability}, where the family of functions $f$ is polynomially negatable under $R$, any function $n: \Theta \to \Theta$ according to Def.~\ref{def:simplenegation} and any solution $\theta$ to this instance, $(\theta, n(\theta))$ is a solution to the instance of \textsc{partial-separability} wrt.~$(J, X, A, B), \Theta, f, R$ and $2 m$. Moreover, this solution can be computed efficiently from $\theta$.
\end{corollary}

\begin{lemma}
	\label{lemma:functionsimplenegatableg}
	For any instance $(J, X, A, B), \Theta, f, R$ of \textsc{separation}, where the family of functions $f$ is polynomially negatable under $R$, any function $n: \Theta \to \Theta$ according to Def.~\ref{def:simplenegation} and the instance of \textsc{partial-separation} wrt.~$(J, X, A, B), \Theta, f$ and $R$, 
	any function $g: \Theta^2 \to \Theta$ such that for any $(\theta, \theta') \in \Theta^2$:
	\begin{equation}
		g(\theta, \theta') \in \argmin_{\theta''\in \{\theta, n(\theta')\}} R(\theta'')
	\end{equation}
	has the following properties: For any feasible solution $(\theta, \theta')\in \Theta^2$ to the instance of \textsc{partial-separation}:
	\begin{enumerate}
		\item $g(\theta, \theta')$ is computable in polynomial time.
		\item $g(\theta, \theta')$ is such that $R(g(\theta, \theta')) \leq (R(\theta) + R(\theta'))/2$.
		\item $g(\theta, \theta')$ is a feasible solution to the instance of \textsc{separation}.
	\end{enumerate}
\end{lemma}

\begin{proof}
	(1) holds by construction of $g$ and the fact that $n: \Theta \to \Theta$ is computable in polynomial time.
	
	(2) holds because
	\begin{align}
	R(g(\theta, \theta')) &	= \min \{ R(\theta), R(n(\theta')) \} \\
	& \overset{\eqref{eq:negation-bound} }{\leq} \min \{ R(\theta), R(\theta') \} \\
	& \leq \left(R(\theta) + R(\theta')\right)/2
	\end{align}
	
	(3) holds by the fact that both $\theta$ and $n(\theta')$ are feasible solutions to the instance $(J, X, A, B), \Theta, f$ of \textsc{separation}.
\end{proof}

\begin{corollary}
	\label{cor:backwardnphardnesssimplenegatable}
	For any instance $(J, X, A, B), \Theta, f, R, m$ of \textsc{separability}, where the family of functions $f$ is polynomially negatable under $R$, any function $n: \Theta \to \Theta$ according to Def.~\ref{def:simplenegation} and any solution $(\theta, \theta')$ to the instance of \textsc{partial-separability} wrt. $(J, X, A, B), \Theta, f, R$ and $2m$, 
	$\theta$ or $n(\theta')$ is a solution to the instance of \textsc{separability} wrt. $(J, X, A, B), \Theta, f, R$ and $m$. Moreover, this solution can be computed efficiently from $(\theta, \theta')$.
\end{corollary}

\begin{theorem}
\label{theorem:hardness-partial-separability}
	For any family $f$ and regularizer $R$ such that $f$ is polynomially negatable under $R$,
	\textsc{separability} $\leq_p$ \textsc{partial-separability}.
\end{theorem}

\begin{proof}
	By Corollaries~\ref{cor:forwardnphardnesssimplenegatable} and \ref{cor:backwardnphardnesssimplenegatable}.
\end{proof}

\begin{theorem}
\begin{enumerate}
	\item For the family $f$ of binary decision trees and $R$ the depth or number of nodes of a tree, \textsc{partial-separability} is \textsc{np}-complete.

	\item For the family $f$ of OBDDs and $R$ the number of interior nodes or width of the diagram, \textsc{partial-separability} is \textsc{np}-complete.
\end{enumerate}
\end{theorem}

\begin{proof}
	These special cases of \textsc{partial-separability} are in \textsc{np} as solutions can be verified in polynomial time. 
	
	(1) This special case of \textsc{partial-separability} is \textsc{np}-hard by Theorem~\ref{theorem:hardness-partial-separability} and the fact that binary decision trees are polynomially negatable under $R$ and the fact that \textsc{separability} for binary decision trees and $R$ is \textsc{np}-hard \citep{Haussler1988}.
	
	(2) This special case of \textsc{partial-separability} is \textsc{np}-hard by Theorem~\ref{theorem:hardness-partial-separability} and the fact that OBDDs are polynomially negatable under $R$ and the fact that \textsc{separability} for OBDDs and $R$ is \textsc{np}-hard \citep{takenaga-2000}.
\end{proof}

\subsection{Disjunctive Normal Forms}

In this section, we establish \textsc{np}-completeness of \textsc{separability} for the family of DNFs regularized by length or depth, by reduction of \textsc{set-cover}. 
This case requires special attention as DNFs are not polynomially negatable \citep{Crama2011}.
Here, we already relate feasible solutions as shown in Fig.~\ref{fig:reduction-dnf-mapping} in a way that will allow us in the next section to transfer an inapproximability bound from \textsc{min-set-cover} to \textsc{partial-separation}.

\begin{definition}[\citet{Haussler1988}]
	For any finite, non-empty set $U$ and any $\Sigma \subseteq 2^{U}$, we use the term \emph{Haussler data} for the Boolean labeled data $D_{U, \Sigma}=(\Sigma, X, A, B)$ such that $B = \{0^\Sigma\}$ and 
	$A = \{ x^u \in \{0,1\}^\Sigma \;|\; u \in U\}$
	such that for any $u \in U$ and any $\sigma \in \Sigma$, we have $x^u_\sigma = 1$ iff $u \in \sigma$.
\end{definition}

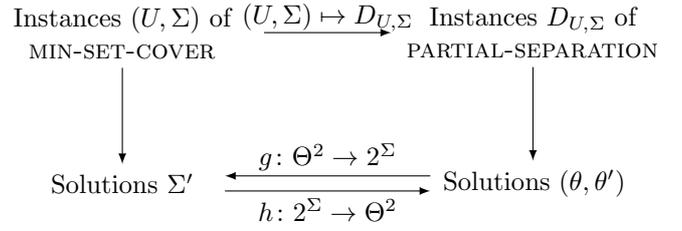
\begin{figure}[t]
	\hspace{-3ex}
	\begin{tikzpicture}[xscale=2.7,yscale=-1]
		\node[text width=23ex, align=center] at (0, 0) (icov) {Instances $(U, \Sigma)$ of \textsc{min-set-cover}};
		\node[text width=23ex, align=center] at (2, 0) (isep) {Instances $D_{U, \Sigma}$ of \textsc{partial-separation}};
		\node[text width=23ex, align=center] at (0, 2) (scov) {Solutions $\Sigma'$};
		\node[text width=23ex, align=center] at (2, 2) (ssep) {Solutions $(\theta,\theta')$};
		\draw[-latex] (icov) -- (isep);
		\node at (1, -0.2) {$(U, \Sigma) \mapsto D_{U, \Sigma}$};
		\draw[-latex] (1.5, 1.9) -- (0.5, 1.9);
		\node at (1, 1.65) {$g \colon \Theta^2 \to 2^{\Sigma}$};
		\draw[-latex] (0.5, 2.1) -- (1.5, 2.1);
		\node at (1, 2.35) {$h \colon 2^{\Sigma} \to \Theta^2$};
		\draw[-latex] (icov) -- (scov);
		\draw[-latex] (isep) -- (ssep);
	\end{tikzpicture}
	\caption{For DNFs, we reduce \textsc{min-set-cover} to \textsc{partial-separation} in an approximation preserving manner by mapping instances of \textsc{min-set-cover} to instances of \textsc{partial-separation} and by relating solutions to these instances by the functions $g$ and $h$.}
	\label{fig:reduction-dnf-mapping}
\end{figure}

\begin{lemma}
	\label{lemma:functionsetcovertodnfs}
	For any instance $(U, \Sigma)$ of \textsc{min-set-cover},
	consider the instance of \textsc{partial-separation} 
	wrt.~the family $f: \Theta \to \{0, 1\}^\Sigma$ of DNFs and $R \in \{R_l, R_d\}$
	and wrt.~the Haussler data $D_{U, \Sigma}$.
	Then, the function $h: 2^{\Sigma} \to \Theta^2$ such that for any $\Sigma'\subseteq \Sigma$, $h(\Sigma') = \left(\theta, \theta'\right)$ with
	\begin{align}
	\theta &= \left\{(\{\sigma\}, \emptyset) \ \vert\ \sigma\in\Sigma'\right\} \label{eq:dnf:theta}\\
	\theta' &= \{(\emptyset, \Sigma')\} \label{eq:dnf:thetaprime}
	\end{align}
	has the following properties: 
	For any feasible solution $\Sigma'$ to the instance of \textsc{min-set-cover}:
	\begin{enumerate}
		\item $h(\Sigma')$ is computable in linear time, $O(|\Sigma'|)$.
		\item $h(\Sigma') = (\theta, \theta')$ is such that $R_l(\theta) + R_l(\theta') = 2\vert \Sigma'\vert$ and $R_d(\theta) + R_d(\theta') = \vert \Sigma'\vert + 1$.
		\item $h(\Sigma')$ is a feasible solution to the instance of \textsc{partial-separation}.
	\end{enumerate} 
\end{lemma}

\begin{proof}
	(1) and (2) hold by construction of $h$.
	(3) Firstly, $f_{\theta}$ and $f_{\theta'}$ are such that for any $x \in \{0, 1\}^\Sigma$:
	\begin{align}
	&f_{\theta}(x) = \sum_{\sigma\in \Sigma'}x_\sigma \label{eq:dnf:ftheta} \\
	&f_{\theta'}(x) = \prod_{\sigma\in \Sigma'} \left(1-x_\sigma\right) \label{eq:dnf:fthetaprime}
	\end{align}
	Secondly, these functions are non-contradicting, i.e.~$f_{\theta} + f_{\theta'} \leq 1$, since $f_{\theta} = 1 - f_{\theta'}$. 	
	Thirdly, $f_{\theta'}(0^\Sigma) = 1$, by \eqref{eq:dnf:fthetaprime}. 
	Moreover, since $\Sigma'$ is a cover of $U$:
	\begin{align}
	&\forall u\in U ~\exists \sigma \in \Sigma'\colon u\in \sigma & \nonumber\\
	\Rightarrow~ &\forall u\in U ~\exists \sigma \in \Sigma'\colon x^u_\sigma = 1 & (\textnormal{Def.~of $D_{U, \Sigma}$})\\
	\Rightarrow~ & \forall u\in U \colon f_{\theta}(x^u) = 1& (\textnormal{Def.~of $\theta$})
	\end{align}
	Thus, $h(\Sigma') = (\theta, \theta')$ is a feasible solution to the instance of \textsc{partial-separation}.
\end{proof}

\begin{corollary}
\label{corollary:forwardnphardnessdnfs}	
For any instance $(U, \Sigma, m)$ of \textsc{set-cover} and any solution $\Sigma'\subseteq \Sigma$ to this instance, 
$h(\Sigma')$ is a solution to the instance of \textsc{partial-separability} 
wrt.~the family $f: \Theta \to \{0, 1\}^\Sigma$ of DNFs, the Haussler data $D_{U, \Sigma}$, the regularizer $R_l$ and the bound $2m$.
The same holds for the regularizer $R_d$ and the bound $m+1$.
Moreover, the solution $h(\Sigma')$ can be computed efficiently from $\Sigma'$.
\end{corollary}

\begin{lemma}
	\label{lemma:functiondnfstosetcover}
	For any instance $(U, \Sigma)$ of \textsc{min-set-cover} and the instance of \textsc{partial-separation} wrt.~the family $f: \Theta \to \{0, 1\}^\Sigma$ of DNFs, $R \in \{R_l, R_d\}$ and the Haussler data $D_{U, \Sigma}$,
	any function $g: \Theta^2 \to 2^\Sigma$ such that for any $(\theta, \theta') \in \Theta^2$,
	\begin{align}
	g(\theta, \theta') \in \argmin_{\Sigma' \in \{\Sigma'_0, \Sigma'_1\}} |\Sigma'| 
	\label{eq:dnf:defg}
	\end{align}
	with
	\begin{align}
	\Sigma_0' & = \bigcup_{(\Sigma_0, \Sigma_1)\in \theta} \Sigma_0 
	\label{eq:dnf:defsigma0}
	\\
	\Sigma_1' & \in \begin{cases}
	\left\{\Sigma_1 \subseteq \Sigma \ \vert\ (\emptyset, \Sigma_1)\in \theta'\right\} & \textnormal{if non-empty}\\
	\left\{\emptyset\right\} & \textnormal{ otherwise}
	\end{cases}
	\label{eq:dnf:defsigma1}
	\end{align}
	has the following properties: 
	For any feasible solution $(\theta, \theta')\in \Theta^2$ to the instance of \textsc{partial-separation}:
	\begin{enumerate}
		\item $g(\theta, \theta')$ is computable in linear time, $O(R_l(\theta) + R_l(\theta'))$.
		\item $g(\theta, \theta')$ is such that $\vert g(\theta, \theta')\vert \leq (R_l(\theta) + R_l(\theta'))/2$ and $\vert g(\theta, \theta')\vert \leq R_d(\theta) + R_d(\theta') - 1$. 
		\item $g(\theta, \theta')$ is a feasible solution to the instance of \textsc{min-set-cover}.
	\end{enumerate}
\end{lemma}

\begin{proof}
	(1) holds because $\Sigma_0'$ and $\Sigma_1'$ can be constructed in time $O(R_l(\theta) + R_l(\theta'))$.
	(2) Firstly:
	\begin{align}
	\vert g(\theta, \theta')\vert 
	=\ & \min\{\vert \Sigma_0'\vert, \vert \Sigma_1'\vert\} & (\textnormal{by}~\eqref{eq:dnf:defg})\\
	\leq\ & \frac{\vert \Sigma_0'\vert + \vert \Sigma_1'\vert}{2} \\
	\leq\ & \frac{R_l(\theta) + R_l(\theta')}{2} & (\textnormal{by}~\eqref{eq:dnf:defsigma0}, \eqref{eq:dnf:defsigma1}) 
	\end{align}
	Secondly:
	\begin{align}
	\vert g(\theta, \theta')\vert 
	=\ & \min\{\vert \Sigma_0'\vert, \vert \Sigma_1'\vert\} & (\textnormal{by}~\eqref{eq:dnf:defg})\\
	\leq\ & |\Sigma_1'| \\
	\leq\ & R_d(\theta') & (\textnormal{by}~\eqref{eq:dnf:defsigma1}) \\
	\leq\ & R_d(\theta) + R_d(\theta') - 1 & (\textnormal{as }1 \leq R_d(\theta)) 
	\end{align}

	(3) We recall from Def.~\ref{def:dnfs} that for any DNF $\theta \in \Theta$, the function $f_\theta$ is such that for any $x \in \{0, 1\}^\Sigma$:
	\begin{align}
	f_{\theta}(x) 
	=\sum_{(\Sigma_0, \Sigma_1)\in \theta} \prod_{\sigma\in \Sigma_0} x_\sigma \prod_{\sigma\in \Sigma_1} \left(1-x_\sigma\right) \enspace .
	\label{eq:lemma2:dnf}
	\end{align}
	Firstly, we show that $\Sigma_0'$ is a feasible solution to the instance of \textsc{min-set-cover} wrt.~$(U, \Sigma)$. 
	On the one hand:
	\begin{align}
	& f_{\theta'}(0^\Sigma) = 1 & (\textnormal{by Def.~of}~D_{U, \Sigma}) \nonumber
	\\
	\Rightarrow\ & f_{\theta}(0^\Sigma) = 0 & (\textnormal{as}~f_{\theta} + f_{\theta'} \leq 1) 
	\\
	\overset{\eqref{eq:lemma2:dnf}}{\Rightarrow}\ & \forall (\Sigma_0, \Sigma_1) \in \theta\colon \Sigma_0 \neq \emptyset \enspace .
	\label{eq:allsigma0nonempty}
	\end{align}
	On the other hand:
	\begin{align}
	& \forall u \in U \colon f_{\theta}(x^u) = 1 \nonumber
	\\
	\overset{\eqref{eq:lemma2:dnf}}{\Rightarrow}\ & \forall u \in U \ 
	\exists (\Sigma_0, \Sigma_1)\in \theta \colon \nonumber \\
	& \quad
		(\forall \sigma \in \Sigma_0\colon x^u_\sigma = 1) \land (\forall \sigma\in \Sigma_1\colon x^u_\sigma = 0) \\
	\overset{\eqref{eq:allsigma0nonempty}}{\Rightarrow}\ & 
	\forall u \in U \ 
	\exists (\Sigma_0, \Sigma_1) \in \theta \ 
	\exists \sigma \in \Sigma_0 \colon \ 
		x^u_\sigma = 1
	\label{eq:existstermmapstoone}
	\\
	\Rightarrow\ & 
	\forall u \in U \ 
	\exists \sigma \in \Sigma_0' \colon \ 
		x^u_\sigma = 1
		\qquad (\textnormal{Def.~of $\Sigma_0'$})
	\\
	\Rightarrow\ & 
	\forall u \in U \ 
	\exists \sigma \in \Sigma_0' \colon \ 
		u \in \sigma
		\qquad (\textnormal{Def.~of}~D_{U, \Sigma})
	\\
	\Rightarrow\ & 
	\bigcup_{\sigma\in \Sigma_0'}\sigma = U
	\end{align}
	Secondly, we show that $\Sigma_1'$ is a feasible solution to the instance of \textsc{min-set-cover} wrt.~$(U, \Sigma)$. 
	On the one hand:
	\begin{align}
	&f_{\theta'}(0^\Sigma) = 1 & (\textnormal{Def.~of } D_{U, \Sigma})\nonumber
	\\
	\overset{\eqref{eq:lemma2:dnf}}{\Rightarrow} \ &\exists (\Sigma_0, \Sigma_1) \in \theta' \colon \Sigma_0 = \emptyset \\
	\Rightarrow \ &\{\Sigma_1\subseteq \Sigma | (\emptyset, \Sigma_1)\in \theta'\} \neq \emptyset\enspace. \label{eq:lemma2:sigmaoneprimenonempty}
	\end{align}
	On the other hand:
	\begin{align}
	&\forall u\in U \colon f_{\theta}(x^u) = 1 \ \qquad (\textnormal{Def.~of } D_{U, \Sigma}) \nonumber
	\\
	\Rightarrow \ & \forall u\in U\colon f_{\theta'}(x^u) = 0 \qquad (\textnormal{as }f_{\theta}\cdot f_{\theta'} = 0)
	\\
	\overset{\eqref{eq:lemma2:dnf}}{\Rightarrow} \ &  \forall u\in U \ \forall(\Sigma_0, \Sigma_1)\in\theta' \colon \nonumber
	\\
	& \quad (\exists \sigma \in \Sigma_0\colon x^u_\sigma = 0) \lor (\exists \sigma \in \Sigma_1\colon x^u_\sigma = 1) 
	\\
	\Rightarrow \ & \forall u\in U\ \exists\sigma\in\Sigma_1'\colon x^u_\sigma = 1 \quad (\textnormal{Def.~of } \Sigma_1' \textnormal{ and }\eqref{eq:lemma2:sigmaoneprimenonempty})
	\\
	\Rightarrow \ & \forall u\in U \ \exists \sigma \in \Sigma_1'\colon u\in \sigma \qquad ~~ (\textnormal{Def.~of } D_{U, \Sigma})
	\\
	\Rightarrow \ &\bigcup_{\sigma\in \Sigma_1'}\sigma = U
	\end{align}
	Thus $g(\theta, \theta')$ is a feasible solution to the instance of \textsc{min-set-cover} wrt.~$(U, \Sigma)$.
	\end{proof}

\begin{corollary}
\label{corollary:backwardnphardnessdnfs}
Let $(U, \Sigma, m)$ an instance of \textsc{set-cover}.

For any solution $(\theta, \theta')$ to the instance of \textsc{partial-separability} wrt.~the family $f: \Theta \to \{0, 1\}^\Sigma$ of DNFs, the Haussler data $D_{U, \Sigma}$, the regularizer $R_l$ and the bound $2m$, $g(\theta, \theta')$ is a solution to the instance of \textsc{set-cover}.

For any solution $(\theta, \theta')$ to the instance of \textsc{partial-separability} wrt.~the family $f: \Theta \to \{0, 1\}^\Sigma$ of DNFs, the Haussler data $D_{U, \Sigma}$, the regularizer $R_d$ and the bound $m + 1$, $g(\theta, \theta')$ is a solution to the instance of \textsc{set-cover}.

In both cases, the solution $g(\theta, \theta')$ can be computed efficiently from $(\theta, \theta')$.
\end{corollary}

\begin{theorem}
\label{theorem:hardness-dnf}
For the family $f$ of DNFs and $R$ the depth or length of a DNF, \textsc{set-cover} $\leq_p$ \textsc{partial-separability}.
\end{theorem}

\begin{proof}
By Corollaries~\ref{corollary:forwardnphardnessdnfs} and \ref{corollary:backwardnphardnessdnfs}.
\end{proof}

\begin{theorem}
For the family $f$ of DNFs and $R$ the depth or length of a DNF, \textsc{partial-separability} is \textsc{np}-complete.
\end{theorem}

\begin{proof}
	This special case of \textsc{partial-separability} is in \textsc{np} as solutions can be verified in polynomial time. 
	It is \textsc{np}-hard by Theorem~\ref{theorem:hardness-dnf} and \textsc{np}-hardness of \textsc{set-cover} \cite{karp-1972}.
\end{proof}

\section{Hardness of Approximation}

\subsection{Polynomially Negatable Functions}

For regularizers $R$ and families $f$ polynomially negatable under $R$, we now reduce \textsc{separation} to \textsc{partial-separation} in an approximation-preserving (AP) manner.

For the family of binary decision trees regularized by length or depth, this and the AP reduction of \textsc{min-set-cover} to \textsc{separation} by \citet{Hancock1996} allow us to transfer the inapproximability bound for \textsc{min-set-cover} by \citet{irit-2014}.
For the family of OBDDs regularized by the number of interior nodes, our reduction of \textsc{separation} to \textsc{partial-separation} and the AP reduction of \textsc{min-set-cover} to \textsc{separation} by \citet{Hirata1996} allow us to transfer the inapproximability bound for \textsc{min-set-cover} by \citet{irit-2014} as well.

For both cases, we conclude that no polynomial-time $o(\ln (|A| + |B| - 1))$-approximation algorithm exists unless \textsc{p=np}.

\begin{lemma}
	\label{lemma:performanceratioscharacterizations}
	Consider 
	any instance $(J, X, A, B), \Theta, f, R$ of \textsc{separation} where the family $f$ is polynomially negatable under $R$, and any solution $\hat\theta_C$ to this instance.
	Consider the instance of \textsc{partial-separation} wrt.~$(J, X, A, B), \Theta, f, R$.
	Moreover, consider any solution $(\hat{\theta}_S,\hat{\theta}_S')$ and feasible solution $(\theta_S, \theta_S')$ to this instance.
	The function $g: \Theta^2 \to \Theta$ from Lemma~\ref{lemma:functionsimplenegatableg} is such that 
	\begin{align}
		\frac{R(g(\theta_S, \theta_S'))}{R(\hat{\theta}_C)} 
			& \leq \frac{R(\theta_S) + R(\theta'_S)}{R(\hat\theta_S) + R(\hat\theta'_S)}
		\enspace .
		\label{eq:lemma:performanceratiotranslationcharacterization}
	\end{align}
\end{lemma}

\begin{proof}
	By Lemma~\ref{lemma:functionsimplenegatableg}, we have
	\begin{equation}
		R(g(\theta_S,\theta_S')) \leq \frac{R(\theta_S) + R(\theta_S')}{2}
		\enspace .
		\label{eq:characterizationnominatorboundg}
	\end{equation}
	Consider the function $h: \Theta \to \Theta^2$ from Lemma~\ref{lemma:functionsimplenegatableh}. 
	For $(\tilde{\theta}_S, \tilde{\theta}_S') = h(\hat{\theta}_C)$ we have, by optimality of $(\hat{\theta}_S, \hat{\theta}_S')$:
	\begin{align} 
		R(\hat{\theta}_S) + R(\hat{\theta}_S') 
		& \leq R(\tilde{\theta}_S) + R(\tilde{\theta}_S') 
		\\
		& \leq 2 R(\hat{\theta}_C) \qquad (\textnormal{by Lemma~\ref{lemma:functionsimplenegatableh}}) 
		\\
		\Rightarrow \hspace{10ex} R(\hat{\theta}_C) 
		& \geq \frac{R(\hat{\theta}_S) + R(\hat{\theta}_S')}{2}
		\\
		\Rightarrow \hspace{4.5ex} \frac{R(g(\theta_S, \theta_S')}{R(\hat{\theta}_C)} 
		& \leq \frac{R(\theta) + R(\theta')}{R(\hat{\theta}) + R(\hat{\theta}')} \qquad\;\; (\textnormal{by } \eqref{eq:characterizationnominatorboundg})
	\end{align}
\end{proof}


\begin{theorem}
	\label{theorem:apx-bdt}
	There is no polynomial-time $o(\ln (|A| + |B| -1))$-approximation algorithm for 
	\textsc{partial-separation} wrt.~the family of binary decision trees regularized by the number of nodes or depth unless \textsc{p=np}.
\end{theorem}

\begin{proof}
	An AP reduction of \textsc{min-set-cover} to \textsc{separation} for binary decision trees regularized by number of nodes or depth is given by \citet{Hancock1996}.
	The reduction of \textsc{separation} to \textsc{partial-separation} from Theorem~\ref{theorem:hardness-partial-separability} is AP according to Lemma~\ref{lemma:performanceratioscharacterizations}, more specifically \eqref{eq:lemma:performanceratiotranslationcharacterization}, and the fact that binary decision trees are polynomially negatable under the number of nodes or depth of a tree (Lemma~\ref{lemma:negatable-bdt-obdd}).
	The composition of these reductions is an AP reduction of \textsc{min-set-cover} to \textsc{partial-separation}.
	By this AP reduction, a polynomial-time $o(\ln (|A| + |B| -1))$-approximation algorithm for \textsc{partial-separation} wrt.~the family of binary decision trees, regularized by the number of nodes or depth,
	implies the existence of a $o(\ln |U|)$-approximation algorithm for \textsc{min-set-cover}.
	Such an algorithm does not exist unless $\textsc{p=np}$, due to \citet{irit-2014}.
\end{proof}

\begin{theorem}
	There is no polynomial-time $o(\ln (|A| + |B| -1))$-approximation algorithm for 
	\textsc{partial-separation} wrt.~the family of OBDDs regularized by the number of interior nodes unless \textsc{p=np}.
\end{theorem}

\begin{proof}
	The proof is analogous to the proof of Theorem~\ref{theorem:apx-bdt}. 
	Here, we employ the AP reduction of \textsc{min-set-cover} to \textsc{separation} for OBDDs, regularized by the number of interior nodes, by \citet{Hirata1996}.
\end{proof}

\subsection{Disjunctive Normal Forms}

For the family of DNFs, regularized by length or depth, we now construct an AP reduction of \textsc{min-set-cover} to \textsc{partial-separation}.
By means of this AP reduction, we transfer directly the inapproximability result of \citet{irit-2014}. 

\begin{lemma}
	\label{lemma:performanceratiossetcoverdnfsandbdts}
	Consider 
	any instance $(U, \Sigma)$ of \textsc{min-set-cover} and 
	any solution $\hat\Sigma$ to this instance.
	Consider 
	the instance of \textsc{partial-separation} wrt.~the Haussler data $D_{U, \Sigma}$, the family of DNFs and with $R \in \{R_l, R_d\}$.
	Moreover, consider any solution $(\hat{\theta},\hat{\theta}')$ and feasible solution $(\theta, \theta')$ to this instance.
	The function $g: \Theta^2 \to 2^\Sigma$ from Lemma~\ref{lemma:functiondnfstosetcover} is such that 
	\begin{align}
		\frac{\vert g(\theta, \theta')\vert}{\vert \hat{\Sigma}\vert}
		\leq
		\frac{R(\theta) + R(\theta')}{R(\hat{\theta}) + R(\hat{\theta}')} 
		\enspace .
		\label{eq:lemma6:performanceratiotranslation}
	\end{align}
\end{lemma}

\begin{proof}
	By Lemma~\ref{lemma:functiondnfstosetcover}, we have for $R_l$:
	\begin{equation}
	\vert g(\theta,\theta')\vert 
		\leq \frac{R_l(\theta) + R_l(\theta')}{2}
	\end{equation}
	Consider the function $h: 2^\Sigma \to \Theta^2$ from Lemma~\ref{lemma:functionsetcovertodnfs}. 
	For $(\tilde{\theta}, \tilde{\theta}') = h(\hat{\Sigma})$, we have, by optimality of $(\hat{\theta}, \hat{\theta}')$:
	\begin{align} 
	R_l(\hat{\theta}) + R_l(\hat{\theta}') 
	& \leq R_l(\tilde{\theta}) + R_l(\tilde{\theta}') 
	\\
	& = 2\vert \hat{\Sigma}\vert \qquad (\textnormal{by Lemma~\ref{lemma:functionsetcovertodnfs}}) 
	\\
	\Rightarrow \hspace{10ex} \vert \hat{\Sigma}\vert 
	& \geq \frac{R_l(\hat{\theta}) + R_l(\hat{\theta}')}{2}
	\\
	\Rightarrow \hspace{4.5ex} \frac{\vert g(\theta, \theta')\vert}{\vert \hat{\Sigma}\vert} 
	& \leq \frac{R_l(\theta) + R_l(\theta')}{R_l(\hat{\theta}) + R_l(\hat{\theta}')} \qquad\;\; (\textnormal{by } \eqref{eq:dnfsnominatorboundg})
	\end{align}
	By Lemma 5, we have for $R_d$:
	\begin{equation}
		\vert g(\theta,\theta')\vert 
		\leq R_d(\theta) + R_d(\theta') - 1
		\label{eq:dnfsnominatorboundg}
	\end{equation}
	Consider the function $h: 2^\Sigma \to \Theta^2$ from Lemma~\ref{lemma:functionsetcovertodnfs}. 
	For $(\tilde{\theta}, \tilde{\theta}') = h(\hat{\Sigma})$, we have, by optimality of $(\hat{\theta}, \hat{\theta}')$:
	\begin{align} 
		R_d(\hat{\theta}) + R_d(\hat{\theta}') 
		& \leq R_d(\tilde{\theta}) + R_d(\tilde{\theta}') 
		\\
		& = \vert \hat{\Sigma}\vert +1\qquad (\textnormal{by Lemma~\ref{lemma:functionsetcovertodnfs}}) 
		\\
		\Rightarrow \hspace{10ex} \vert \hat{\Sigma}\vert 
		& \geq R_d(\hat{\theta}) + R_d(\hat{\theta}') - 1
		\\
		\Rightarrow \hspace{4.5ex} \frac{\vert g(\theta, \theta')\vert}{\vert \hat{\Sigma}\vert} \\
		\frac{\vert g(\theta, \theta')\vert}{\vert \hat{\Sigma}\vert} &\leq \frac{R_d(\theta) + R_d(\theta') - 1}{R_d(\hat{\theta}) + R_d(\hat{\theta}')-1} \\
		& \leq \frac{R_d(\theta) + R_d(\theta')}{R_d(\hat{\theta}) + R_d(\hat{\theta}')} \qquad\;\; (\textnormal{by } \eqref{eq:dnfsnominatorboundg})
	\end{align}
\end{proof}

\begin{theorem}
	There is no polynomial-time $o(\ln (|A| + |B| -1))$-approximation algorithm for 
	\textsc{partial-separation} wrt.~the family of DNFs regularized by length or depth, unless \textsc{p=np}.
\end{theorem}

\begin{proof}
	An AP reduction of \textsc{min-set-cover} to \textsc{partial-separation} wrt.~the family of DNFs and wrt.~$R_l$ or $R_d$ is given by Lemma~\ref{lemma:performanceratiossetcoverdnfsandbdts}.
	By this AP-reduction, the existence of a polynomial time $o(\ln (|A| + |B| - 1))$-approximation algorithm for \textsc{partial-separation} wrt.~the family of DNFs and wrt.~$R_l$ or $R_d$ implies the existence of a polynomial time $o(\log |U|)$-approximation algorithm for \textsc{min-set-cover}. 
	Such an algorithm does not exists unless $\textsc{p=np}$, due to \citet{irit-2014}.
\end{proof}

\section{Conclusion}
We hypothesized that \textsc{partial-separation} was an easier problem than \textsc{separation}.
For families of Boolean functions polynomially negatable under regularizers, we have refuted this hypothesis by reducing \textsc{separation} to \textsc{partial-separation} in an approximation-preserving manner.
This has allowed us to transfer known inapproximability results for the special case of binary decision trees and OBDDs.
For families of Boolean functions not polynomially negatable under regularizers, we have not refuted the hypothesis. 
For DNFs, however, we have reduced \textsc{min-set-cover} to \textsc{partial-separation} in an approximation-preserving manner and have thus established the tightest inapproximability bound known for \textsc{separation} also for \textsc{partial-separation}.
We conclude for all these cases that the learning of partial Boolean functions is at least as hard as the learning of (total) Boolean functions.
For other families of functions not polynomially invertible under regularizers, the hypothesis is non-trivial and still open.
While all our proofs are straight-forward and do not introduce new techniques, our theorems can inform the discussion of heuristic approaches to learning partial functions.

\bibliography{arxiv}
\bibliographystyle{plainnat}

\end{document}